\documentclass[pmlr]{jmlr}

\usepackage{booktabs}
\usepackage{siunitx}
\usepackage{graphicx}
\usepackage{amsmath,amssymb}
\usepackage{setspace}
\usepackage{microtype}
\usepackage{xspace}
\usepackage{tikz}
\usepackage{multirow}
\usepackage{makecell}
\usepackage{algorithm2e}

\usepackage[separate-uncertainty=true, group-digits=integer]{siunitx}
\ExplSyntaxOn
\NewDocumentCommand{\bsc}{m}
 {
  \textbf
   {
    \tl_map_inline:nn { #1 }
     {
      \regex_match:nnTF { [a-z] } { ##1 }
       { {\scriptsize \MakeUppercase{##1}} }
       { ##1 }
     }
   }
 }
\ExplSyntaxOff

\ExplSyntaxOn
\NewDocumentCommand{\autobsc}{m}
 {
  \tl_map_inline:nn { #1 }
   {
    \regex_match:nnTF { [a-z] } { ##1 }
     { {\scriptsize \MakeUppercase{##1}} }
     { ##1 }
   }
 }
\ExplSyntaxOff

\jmlrvolume{TBD}
\jmlryear{2026}
\jmlrworkshop{Probabilistic Graphical Models (PGM)}
\editor{Gustau Camps-Valls, Manuele Leonelli and Gherardo Varando}

\newcommand{\fedrcd}{\textsc{FedRCD}\xspace}
\newcommand{\fedrcdni}{\textsc{FedRCD-NI}\xspace}
\newcommand{\fedrcddef}{\textsc{FedRCD-Def}\xspace}
\newcommand{\fedrcdx}{\textsc{FedRCD-X}\xspace}
\newcommand{\fedishc}{\textsc{FedISHC}\xspace}
\newcommand{\fedishcfour}{\textsc{FedISHC}($\tau^{(4)}$)\xspace}
\newcommand{\fedhc}{\textsc{FedHC}\xspace}

\newcommand{\lingam}{\textsc{LiNGAM}\xspace}
\newcommand{\hclingam}{\textsc{HC-LiNGAM}\xspace}
\newcommand{\directlingam}{\textsc{DirectLiNGAM}\xspace}
\newcommand{\hc}{\textsc{HC}\xspace}
\newcommand{\icalingam}{\textsc{ICA-LiNGAM}\xspace}
\newcommand{\AncAcc}{\textsc{AncAcc}\xspace}
\newcommand{\er}{Erd\H{o}s-R\'enyi\xspace}
 
\newcommand{\bX}{\mathbf{X}}
\newcommand{\bB}{\mathbf{B}}
\newcommand{\bE}{\mathbf{E}}
\newcommand{\bA}{\mathbf{A}}

\newcommand{\bI}{\mathbf{I}}
\newcommand{\Sig}{\boldsymbol{\Sigma}}

\newcommand{\R}{\mathbb{R}}
\newcommand{\G}{\mathcal{G}}

\title{Federated Causal Discovery via Regression-Directed Cumulants}

\author{%
  \Name{Pablo Torrijos} \Email{pablo.torrijos@uclm.es}\\
  \addr Departamento de Sistemas Inform\'aticos\\
  Universidad de Castilla-La Mancha\\
  Albacete, Spain
  \AND
  \Name{Fabio Stella} \Email{fabio.stella@unimib.it}\\
  \addr Dipartimento di Informatica, Sistemistica e Comunicazione\\
  Universit\`a degli Studi di Milano Bicocca\\
  Milano, Italy
  \AND
  \Name{Jos\'e A. G\'amez} \Email{jose.gamez@uclm.es}\\
  \addr Departamento de Sistemas Inform\'aticos\\
  Universidad de Castilla-La Mancha\\
  Albacete, Spain
  \AND
  \Name{Jos\'e M. Puerta} \Email{jose.puerta@uclm.es}\\
  \addr Departamento de Sistemas Inform\'aticos\\
  Universidad de Castilla-La Mancha\\
  Albacete, Spain
}

\begin{document}

\maketitle

%
%
\begin{abstract}
In this paper we study linear non-Gaussian acyclic models (LiNGAM) when used in federated environments. These causal models allow one to go beyond Markov equivalence. However, in many domains data are scarce, and increasing the sample size by centralising data from different clients is not advisable due to regulations such as the General Data Protection Regulation (GDPR). The federated environment offers an attractive option to balance privacy and causal discovery accuracy. Unfortunately, the standard centralised estimator in the LiNGAM setting, i.e., DirectLiNGAM, cannot be straightforwardly federated. Higher-order cumulant tensors offer a way around this obstacle: they depend only on the joint distribution of the variables involved and add exactly across independent sample groups, so a single communication round suffices in horizontal, vertical, and hybrid partitions.
However, FedISHC, i.e., the current federated method along these lines,  breaks down under near-symmetric noise. To overcome the above limitation, we introduce the FedRCD family of causal discovery algorithms, and investigate three variants that trade off communication rounds against algebraic noise; two of them are exact federated counterparts of the centralised high-order cumulant (HC) and HC-LiNGAM algorithms, and the single-round variants further effectively support exact unlearning at any granularity, from a single observation to a whole client. Numerical experiments show that at sample sizes typical of real deployments, the entire cumulant-based federated family does not actually rank variables by the population asymmetry that the scores encode at zero. It ranks them by a variance ladder induced by the DAG along its directed paths, the cumulant counterpart of varsortability. Marginal standardisation collapses every cumulant method to near-random ordering, while scale-invariant DirectLiNGAM, not federable under this protocol, is unaffected.
\end{abstract}

\begin{keywords}
    Federated causal discovery; horizontal, vertical, and hybrid federation; \lingam; higher-order cumulants; federated unlearning.
\end{keywords}

\sloppy

%
%
\section{Introduction} \label{sec:introduction}
Causal structure learning from observational data drives applications in genomics~\citep{TejadaLapuerta2025}, protein signalling~\citep{Zhai2025}, econometrics~\citep{moneta11}, or epidemiology~\citep{Ferrari2022,Wang2022}. However, in some domains, especially in healthcare, data are scarce and the only option to improve sample size is to combine data from patients cohorts across multiple hospitals and/or research centres. While this option serves the purpose of increasing the sample size and thus improving the power of causal learning, it brings severe concerns about data privacy. Privacy regulations such as the General Data Protection Regulation impose strict constraints on sharing raw personal data outside the institution where they were collected. Federated learning~\citep{mcmahan2017,Kairouz2021} addresses the tension between privacy and causal learning power: clients keep their data local and exchange only aggregated statistics through a central server. The federated setting itself splits into three partitioning regimes~\citep{Zhang_survey_fl_KBS_2021}: \emph{horizontal} (same variables, different samples), \emph{vertical} (same samples, different variables), and \emph{hybrid}.

Standard approaches to causal discovery have been adapted to the federated setting. Constraint-based federated methods such as federated PC variants~\citep{Wang2023fed,Huang2023} rely on conditional independence tests over the full conditioning set, while continuous optimisation procedures~\citep{Ng2022fed} rely on algebraic score functions. Both families generally recover only Markov equivalence classes, and both struggle in vertical and hybrid regimes: when the variables required for a conditional independence test are split across clients, the test cannot be evaluated without pooling data, and partial-overlap workarounds introduce spurious edges that propagate through aggregation.
An option to go beyond Markov equivalence and recover a causal ordering is given by the \lingam framework~\citep{shimizu2006,shimizu2011} under the assumption of non-Gaussian exogenous noise. Unfortunately, its state-of-the-art estimator, \directlingam~\citep{shimizu2011}, relies on nonparametric independence tests applied to centralised data, and cannot be straightforwardly federated. Recent centralised methods \citep{Chen2025} replace those tests with closed-form pairwise scores based on higher-order cumulants. This is the natural primitive for federation as cumulants depend only on the joint distribution of the variables involved~\citep{Brillinger2001}, so missing variables on a client do not bias the entries the client can compute (this handles vertical partitioning). They are also additive over independent sample groups~\citep{Speed1983}, so per-client raw-moment sums add up to the global tensor (this handles horizontal partitioning). The protocol is identical for horizontal, vertical, and hybrid partitions, runs in a single round, and supports exact federated unlearning by subtracting specific raw moments.
The current federated method along these lines, \fedishc~\citep{fedishc2026}, instantiates this protocol with third-order cumulants and runs sequential deflation on the server. Its limitation under symmetric noise is expected: its identification score and deflation coefficient both divide by the candidate source's skewness, which is zero for symmetric distributions (Proposition~\ref{prop:cdl}). Fixing this requires more than swapping the third-order score for a fourth-order one: \fedishc updates only third-order arrays at each iteration, so a fourth-order score evaluated on those arrays is invariant under the iteration and behaves like a single-pass ranking. Because identification and deflation must be redesigned together, we introduce the \fedrcd family to solve the symmetric-noise failure by pairing a fourth-order source criterion with a stable covariance-based deflation coefficient whose denominator is bounded away from zero by construction.

There is, however, a less flattering question to ask of any cumulant-based estimator. The asymmetry that the third- and fourth-order pairwise scores encode vanishes for true sources at the population level, which is precisely how the theory identifies them. A relevant question we ask is as follows: at the sample sizes that real federated deployments actually see, what dominates the ranking? We find that the dominant signal is a variance ladder induced by the DAG along its directed paths. Under \lingam, descendants accumulate variance from their ancestors, the marginal cumulants inherit that scale at every order, and the row sums that drive identification across the entire cumulant family, federated or centralised, line up almost perfectly with depth. This is the cumulant counterpart of \emph{varsortability}~\citep{reisach2021nips}, the same scale signal that has been documented to drive continuous optimisation methods such as NOTEARS~\citep{Zheng2018NOTEARS}. Marginal standardisation removes the ladder and collapses every cumulant-based method to near-random ranking, including the centralised baselines \hc and \hclingam; \directlingam, scale-invariant by construction,~is~unaffected.

The main contributions of this paper are the following: 
1) we formalise the symmetric-noise limit of \fedishc and explain why a fourth-order score on its own does not repair it; 
2) we introduce the \fedrcd family, pairing a fourth-order source criterion with a stable covariance-based deflation coefficient, where three variants trade off communication rounds against algebraic noise;
3) we show that the entire cumulant-based federated family ranks variables by a variance ladder induced by the DAG. This places it in the scale-dependent regime that~\citet{reisach2021nips} identified for MSE-based continuous methods;
and finally 4) we provide extensive empirical evidence across \er DAGs under eleven noise families and eight \texttt{bnlearn} BN repository topologies.

The rest of the paper is organised as follows: Section~\ref{sec:background} reviews \lingam, fourth-order identification, and the federated setting. Section~\ref{sec:methods} develops the \fedrcd family. Section~\ref{sec:experiments} presents the empirical evaluation including the stratification diagnostic. Section~\ref{sec:conclusion} concludes.

%
%
\section{Background and Problem Formulation} \label{sec:background}

\subsection{The LiNGAM Model} \label{sec:lingam}
Let $\bX = (x_1, \ldots, x_p)^\top \in \R^p$ be a vector of observed variables. The \lingam framework~\citep{shimizu2006,shimizu2011} models their causal structure as $\bX = \bB\bX + \bE$, where $\bB \in \R^{p \times p}$ is a matrix of causal coefficients that can be permuted to strictly lower-triangular form, and $\bE = (e_1, \ldots, e_p)^\top$ is a vector of mutually independent, non-Gaussian disturbances with $\mathbb{E}[e_i] = 0$ and $\mathbb{E}[e_i^2] = \sigma_i^2 > 0$. Solving for $\bX$ gives the mixing form $\bX = \bA\bE$ with $\bA = (\bI - \bB)^{-1}$, where $\bI$ denotes the $p \times p$ identity matrix. \lingam offers a stronger guarantee than constraint-based methods: as long as at most one disturbance is Gaussian, $\bB$ is uniquely identified from the joint distribution of $\bX$ alone, that is, the entire DAG is recovered rather than only its Markov equivalence class~\citep{shimizu2006}. The required assumptions are linearity, acyclicity, mutual independence of noise terms, and no unobserved confounders; heteroscedastic noise is allowed.

Two classical \lingam estimators are \icalingam~\citep{shimizu2006}, which recovers the causal order via Independent Component Analysis (ICA) on the mixing matrix, and \directlingam~\citep{shimizu2011}, which does so iteratively via regression and nonparametric independence tests between candidate sources and residuals. Both achieve strong structural recovery at high computational cost, and neither admits federation by aggregated statistics alone: ICA operates on the raw data matrix, and the kernel independence tests need joint access to the variables they are testing.

\subsection{Higher-Order Cumulants and Source Identification} \label{sec:cumulants}
The $m$-th order marginal cumulant of $x_i$ is denoted $\kappa_m(x_i)$~\citep{Brillinger2001}, and $\kappa_{m_1,m_2}(x_i, x_j)$ denotes the joint cumulant with $m_1$ copies of $x_i$ and $m_2$ copies of $x_j$. We use $\kappa_\bullet$ for population cumulants and $\hat{\kappa}_\bullet$ for their empirical estimates. We use the third- and fourth-order self-cumulants $\kappa_3, \kappa_4$ and the joint cumulants $\kappa_{1,2}, \kappa_{2,1}, \kappa_{1,3}, \kappa_{2,2}, \kappa_{3,1}$, all of which are particular instances of this notation. Cumulants are multilinear, additive for independent variables, and vanish at order $\geq 3$ for Gaussian variables. Non-Gaussianity breaks the directional symmetry of joint cumulants and is what supplies the statistical signal that distinguishes cause from effect.
The fourth-order pairwise asymmetry score~\citep{Chen2025}~is
\begin{equation} \label{eq:tau4}
    \tau_{ij} = \bigl|\kappa_4(x_i)\,\kappa_{1,3}(x_i, x_j) - \kappa_{2,2}(x_i, x_j)\,\kappa_{3,1}(x_i, x_j)\bigr|.
\end{equation}
In the population limit, $\tau_{ij} = 0$ if and only if $x_i$ is an ancestor of $x_j$ (or the two variables are independent), and $\tau_{ij} > 0$ otherwise~\citep[Theorems~2--3]{Chen2025}. A \emph{source node} of the DAG is a variable with no incoming edges; $x_s$ is a source if and only if $\sum_{j \neq s} \tau_{sj} = 0$. The source at each step of an ordering algorithm is found by solving the optimisation~problem
\begin{equation} \label{eq:source}
    s = \arg\min_{i \in U} \sum_{j \in U,\, j \neq i} \tau_{ij},
\end{equation}
where $U \subseteq \{1, \ldots, p\}$ collects the indices of variables not yet placed in the partial order (the \emph{active set}). High-order cumulant \hc~\citep{Chen2025} identifies each source via~\eqref{eq:source}, deflates by ordinary least squares OLS ($x_j \leftarrow x_j - \hat{\beta}_{js}\,x_s$, $\hat{\beta}_{js} = \hat{\Sigma}_{js}/\hat{\Sigma}_{ss}$), and repeats at $O(np^3)$ total cost. \hclingam~\citep{Chen2025} computes the $\tau$ matrix once and sorts globally by the row sums $T_\tau(x_i) = \sum_{j \neq i} \tau_{ij}$ at $O(np^2)$ cost. Both are centralised reference points for the methods of Section~\ref{sec:methods}.

\subsection{Federated Causal Discovery} \label{sec:federated}
Data are distributed across $K$ clients that cannot share raw observations. Client $k$ holds dataset $\mathcal{D}_k$ over variable set $\bX_k \subseteq \bX$ with $n_k$ samples, where $\bX = \bigcup_{k=1}^K \bX_k$ and $N = \sum_{k=1}^K n_k$. Following~\citet{fedishc2026}, we require that for every pair $(x_i, x_j)$ at least one client holds observations for both. The condition is milder than asking for a common complete variable set, and is what enables coverage of the full joint cumulant tensor; pair-coverage gaps would leave the corresponding cumulant entries unidentifiable.

Each joint cumulant $\kappa_{m_1, m_2}(x_i, x_j)$ is a function of the bivariate distribution of $(x_i, x_j)$ alone, so clients that do not observe both variables contribute nothing to that entry. Cumulants are additive over independent sample groups~\citep{Speed1983}, so the global cumulant tensor (the array indexed by all required pairs and orders) is recovered exactly from per-client raw-moment sums by a weighted sum, regardless of whether the partition is horizontal, vertical, or hybrid. Each client transmits $O(p_k^2)$ floats with $p_k = |\bX_k|$, and the server pools these into global estimates that are numerically identical to centralised computation on the full dataset. Neither property holds for residuals or kernel-based independence tests on raw data, which therefore fall outside this protocol.

\paragraph{\fedishc.}
\fedishc~\citep{fedishc2026} aggregates third-order cumulants in a single round and runs sequential deflation on the server. Sources are identified by the third-order~score
\begin{equation} \label{eq:tau3}
    \tau^{(3)}_{ij} = \bigl|\kappa_3(x_i)\,\kappa_{1,2}(x_i,x_j) - \kappa_{2,1}(x_i,x_j)\,\kappa_{1,2}(x_j,x_i)\bigr|.
\end{equation}
The causal influence of an identified source $x_s$ on each remaining variable $x_j$ is estimated as
\begin{equation} \label{eq:alpha_ishc}
    \hat{\alpha}_{js} = \frac{\hat{\kappa}_{2,1}(x_j,\, x_s)}{\hat{\kappa}_3(x_s)},
\end{equation}
and the third-order cumulant arrays are updated via
\begin{equation} \label{eq:ishc_update}
    \hat{\kappa}_3(x_j)' = \hat{\kappa}_3(x_j) - \hat{\alpha}_{js}^3\,\hat{\kappa}_3(x_s).
\end{equation}
Both~\eqref{eq:tau3} and~\eqref{eq:alpha_ishc} carry $\kappa_3(x_s)$ in their denominator, which Section~\ref{sec:pathologies} exploits to formalise the limit. \citet{fedishc2026} also introduce \fedhc, a no-deflation variant that uses~\eqref{eq:tau3} and sorts variables by row sums in a single pass. \fedhc inherits the symmetric-noise weakness of $\tau^{(3)}$, but its absence of deflation prevents error compounding across the $p-1$ steps.

\subsection{Limitations of \fedishc under symmetric noise} \label{sec:pathologies}
Under symmetric noise the third self-cumulant $\kappa_3(x_s)$ vanishes in population, and at finite samples it is dominated by sampling fluctuation. \fedishc places this quantity in the denominator of both its identification score~\eqref{eq:tau3} and its deflation coefficient~\eqref{eq:alpha_ishc}. The consequences for accuracy under symmetric noise follow directly from this. We record the variance bound on $\hat{\alpha}_{js}$ for completeness:

\begin{proposition}[Symmetric-noise variance bound] \label{prop:cdl}
    Let $\hat{\kappa}_3(x_s)$ and $\hat{\kappa}_{2,1}(x_j, x_s)$ be unbiased aggregated estimators~\citep{schefczik2019} based on $N$ total samples. Assume all moments of $e_s$ up to order six are finite, and consider the regime $\kappa_3(x_s) \neq 0$. By the delta method applied to $f(a,b) = a/b$,
    \begin{equation} \label{eq:cdl_variance}
        \mathrm{Var}[\hat{\alpha}_{js}] = \frac{\mathrm{Var}[\hat{\kappa}_{2,1}]}{\kappa_3(x_s)^2} + \frac{\kappa_{2,1}(x_j,x_s)^2\,\mathrm{Var}[\hat{\kappa}_3(x_s)]}{\kappa_3(x_s)^4} - \frac{2\kappa_{2,1}(x_j,x_s)}{\kappa_3(x_s)^3}\,\mathrm{Cov}[\hat{\kappa}_{2,1},\hat{\kappa}_3] + O(N^{-2}).
    \end{equation}
    Under the \lingam model, $\kappa_{2,1}(x_j,x_s) = \beta_{js}^2\kappa_3(x_s)$, so all three terms scale as $\kappa_3(x_s)^{-2}/N$ and the leading order is $\Theta(\kappa_3(x_s)^{-2}/N)$. Along any sequence of \lingam distributions with $\kappa_3(x_s) \to 0$ at fixed $\beta_{js}$, the variance bound diverges.
\end{proposition}
\begin{proof}
    See Appendix~\ref{app:proofs}.
\end{proof}

The fourth-order score $\tau^{(4)}$ in~\eqref{eq:tau4} avoids the issue, since $\kappa_4(x_s) \neq 0$ for every standard non-Gaussian distribution. Replacing only the score, however, is not enough. Let \fedishcfour denote the algorithm using $\tau^{(4)}$ for identification while retaining the third-order deflation from Equations~\eqref{eq:alpha_ishc} and \eqref{eq:ishc_update}. The deflation in~\eqref{eq:ishc_update} updates only $\kappa_3$ arrays. The score $\tau^{(4)}$ depends on $\kappa_4$ and on the joint cumulants $\kappa_{1,3}, \kappa_{2,2}, \kappa_{3,1}$, and none of these is touched by~\eqref{eq:ishc_update}. Every iteration of \fedishcfour therefore evaluates $\tau^{(4)}$ on the same matrix, restricted to the current active set. Identification and deflation have to be repaired together.

%
%
\section{The \fedrcd Family} \label{sec:methods}
We propose \fedrcd (Federated Regression-Directed Cumulants), a family of federated \lingam estimators that addresses both issues raised in Section~\ref{sec:pathologies}.

\subsection{OLS Deflation Coefficient} \label{sec:fedrcd}
The instability of $\hat{\alpha}_{js}$ in~\eqref{eq:alpha_ishc} comes from its denominator $\hat{\kappa}_3(x_s)$. Ordinary least squares supplies a replacement with a bounded denominator. When $x_s$ is the current source, it has no parents under the \lingam model, so $\Sigma_{js} = b_{js}\Sigma_{ss}$ exactly and the OLS coefficient $\hat{\beta}_{js} = \hat{\Sigma}_{js}/\hat{\Sigma}_{ss}$ is consistent for the same structural parameter $b_{js}$ that $\hat{\alpha}_{js}$ targets. Its denominator $\Sigma_{ss} = \mathrm{Var}(x_s)$ is strictly positive for any non-degenerate variable, whatever the noise distribution. Both $\hat{\Sigma}_{js}$ and $\hat{\Sigma}_{ss}$ are entries of the aggregated $\hat{\Sig}$, so $\hat{\beta}_{js}$ is computed once on the server, never on raw data, and takes the same value as in the centralised case whether or not the clients are IID. The residual $r_j = x_j - \beta_{js}\,x_s$ is orthogonal to $x_s$ in second order; structure at higher orders is handled by the closed-form cumulant updates of Appendix~\ref{app:updates}.

\begin{proposition}[Stability of the OLS deflation coefficient] \label{prop:stability}
    Assume $\mathbb{E}[x_i^4]\!<\!\infty, \forall i$, and $\Sigma_{ss}\!=\!\mathrm{Var}(x_s)\!>\!0$. Then
    \begin{equation} \label{eq:beta_variance}
        \mathrm{Var}[\hat{\beta}_{js}] = \frac{\mathrm{Var}[\hat{\Sigma}_{js}]}{\Sigma_{ss}^2} + \frac{\Sigma_{js}^2\,\mathrm{Var}[\hat{\Sigma}_{ss}]}{\Sigma_{ss}^4} - \frac{2\Sigma_{js}}{\Sigma_{ss}^3}\,\mathrm{Cov}[\hat{\Sigma}_{js},\hat{\Sigma}_{ss}] + O(N^{-2}) = O(N^{-1}),
    \end{equation}
    uniformly over noise distributions with $\mathrm{Var}(x_s) > 0$.
\end{proposition}
\begin{proof}
    See Appendix~\ref{app:proofs}.
\end{proof}

The cumulant ratio $\kappa_3(x_s)^{-2}/N$ of Proposition~\ref{prop:cdl} is replaced by a constant that depends only on second-order moments. Both estimators are consistent for $b_{js}$; the difference is purely numerical. The \fedrcd family uses $\tau^{(4)}$ for identification and $\hat{\beta}_{js}$ wherever a deflation coefficient is required. A second consequence matters in Section~\ref{sec:stratification}: the deflation step does not inject divergent noise into the cumulant arrays, so the depth ordering of variances and cumulants induced by the DAG survives the iteration.

\subsection{The \fedrcd-(NI/Def/X) Variants} \label{sec:variants}
The three variants (Algorithm~\ref{alg:fedrcd}) share the $\tau^{(4)}$ criterion, the $\hat{\beta}_{js}$ coefficient, and the client-side aggregation protocol. They differ in where deflation happens and how many rounds it requires. All three apply to horizontal, vertical, and hybrid federation, and recover $\hat{\bB}$ from the pre-deflation $\hat{\Sig}$ via adaptive Lasso with fixed $\lambda = 0.01\bar{\sigma}$ once the order is determined.

\fedrcdni computes $\tau^{(4)}$ once and sorts variables by the row sums $T_\tau(x_i) = \sum_{j \neq i} \tau_{ij}$ in a single pass at $O(p^2)$ server cost. There is no deflation step: by Theorem~4 of~\citet{Chen2025}, $T_\tau(x_i) < T_\tau(x_j)$ in population whenever $x_i$ is a predecessor of $x_j$, so the row-sum sort recovers the true order. At finite samples and on dense graphs, unresolved confounding adds noise. \fedrcdni is the exact federated counterpart of \hclingam~\citep{Chen2025} and supports \emph{exact instance-level federated unlearning}. Raw moments are additive over independent samples, hence also subtractive: given the raw-moment contribution of any subset to be forgotten (a single observation, a cohort within a client, or an entire client), the server subtracts it from the global aggregates and rescales the totals by $(N - n_{\text{rm}})^{-1}$, recovering exactly the statistics that would have been obtained had those samples never participated. In vertical and hybrid regimes the same procedure applies pair by pair, with the server keeping per-pair counts $N_{ij}$ to rescale each entry of the cumulant tensor. The one-message protocol therefore handles the full spectrum of General Data Protection Regulation right-to-erasure requests, from a single individual withdrawing consent to an entire institution leaving the federation, without retraining.

\fedrcddef adds algebraic deflation on the server after each source removal. Multilinearity of cumulants applied to the residual $r_j = x_j - \beta_{js}\,x_s$ yields closed-form updates of all aggregated arrays (Appendix~\ref{app:updates}; this extends Lemma~2 of~\citet{fedishc2026} from third to fourth order and from $\hat\alpha$ to the stable $\hat\beta$). The server applies these updates in place at $O(p^3)$ total cost, still in a single round. The deflation coefficient is stable by Proposition~\ref{prop:stability}, and the updates are exact at population level. At finite samples, however, every algebraic step injects estimation noise that compounds across $p-1$ updates. Exact federated unlearning is preserved because the protocol remains a one-round exchange.

\begin{algorithm2e}[htbp]
    \caption{\fedrcd Family} \label{alg:fedrcd}
    \footnotesize\setstretch{0.95}
    \DontPrintSemicolon \LinesNumbered \SetCommentSty{textnormal}
    \KwIn{Global stats $\hat{\kappa}_\bullet$, $\hat{\Sig}$ aggregated from all clients; set $U = \{1,\ldots,p\}$; \textbf{mode} $\in \{\text{\textsc{NI,Def,X}}\}$}
    \KwOut{Causal order $K$; matrix $\hat{\bB}$}
    $\hat{\Sig}_0 \leftarrow \hat{\Sig}$\tcp*{Save original covariance for $\hat{\bB}$ estimation}
    $K \leftarrow [\,]$\;
    \eIf{\normalfont\textbf{mode} $= \mathrm{NI}$}{
        Compute $\tau_{ij}$ for all $i \neq j$ via~\eqref{eq:tau4}\;
        $K \leftarrow \operatorname{argsort}\bigl(\textstyle\sum_{j \neq i}\tau_{ij}\bigr)_{i=1}^{p}$\tcp*{Single pass, no deflation}
    }{
        \While{$|U| > 1$}{
            Compute $\tau_{ij}$ for $i,j \in U$ via~\eqref{eq:tau4}\;
            $s \leftarrow \arg\min_{i \in U}\sum_{j \in U,\,j \neq i}\tau_{ij}$\tcp*{Identify source}
            Append $s$ to $K$\;
            $\hat{\beta}_{js} \leftarrow \hat{\Sigma}_{js}/\hat{\Sigma}_{ss}$ for each $j \in U \setminus \{s\}$\tcp*{OLS coefficient}
            \uIf{\normalfont\textbf{mode} $= \text{\textsc{Def}}$}{
                Update $\hat{\kappa}_\bullet,\,\hat{\Sig}$ server-side via~\eqref{eq:cov_update}--\eqref{eq:fedrcd_c21}\tcp*{Appendix \ref{app:updates}}
            }
            \ElseIf{\normalfont\textbf{mode} $= \mathrm{X}$}{
                Server broadcasts $(s,\,\hat{\boldsymbol{\beta}})$ to clients\;
                Clients: $x_j \leftarrow x_j - \hat{\beta}_{js}\,x_s$ for $j \in U\setminus\{s\}$; drop $x_s$\;
                Clients send fresh raw moments; server re-aggregates $\hat{\kappa}_\bullet,\,\hat{\Sig}$
            }
            $U \leftarrow U \setminus \{s\}$\;
        }
        Append remaining element of $U$ to $K$\;
    }
    $\hat{\bB} \leftarrow \text{AdaptiveLasso}(\hat{\Sig}_0,\,K)$\tcp*{Edge weights from original $\hat{\Sig}$}
\end{algorithm2e}

\fedrcdx pushes deflation back to the clients to avoid algebraic accumulation. In each of $p-1$ rounds, the server identifies the source via~\eqref{eq:source}, broadcasts the OLS coefficients, and each client deflates its local data, drops the identified source, and returns fresh sufficient statistics. Cumulants are recomputed from actual residuals at every round rather than approximated algebraically, and no approximation error accumulates. \fedrcdx is the federated counterpart of \hc~\citep{Chen2025}. The price is $p-1$ rounds at $O(p^3)$ total payload. Exact unlearning is no longer available, since each round depends on coefficients derived from the previous aggregate, and removing a client retroactively would invalidate every subsequent round.

Table~\ref{tab:complexity} summarises all methods. For single-round protocols, the total payload equals the per-round payload; for \fedrcdx it equals the per-round payload times $p-1$.

\begin{table}[htbp]
    \caption{Algorithm comparison. $M$: maximum kernel rank in \directlingam.}
    \label{tab:complexity}
    \centering
    \resizebox{\linewidth}{!}{%
    \begin{tabular*}{1.5\linewidth}{@{\extracolsep{\fill}}@{\hspace{0.5em}}cllllll@{\hspace{0.5em}}}
        \toprule
        & \bsc{Method} & \bsc{Deflation} & \bsc{Client} & \bsc{Comm./r} & \bsc{Server} & \bsc{Rounds} \\
        \midrule
        \multirowcell{3}{\rotatebox[origin=c]{90}{\textsc{Centr.}}}
        & \directlingam~\citep{shimizu2011} & Kernel independence test & $O(np^3M^2 + p^4M^3)$ & n/a & n/a & n/a \\
        & \hc~\citep{Chen2025}  & OLS on data & $O(np^3)$ & n/a & n/a & n/a \\
        & \hclingam~\citep{Chen2025} & None  & $O(np^2)$ & n/a & n/a & n/a \\
        \midrule
        \multirowcell{5}{\rotatebox[origin=c]{90}{\textsc{Federated}}}
        & \fedhc~\citep{fedishc2026}   & None  & $O(n_k p^2)$ & $O(p^2)$ & $O(p^2)$ & 1  \\
        & \fedishc~\citep{fedishc2026} & $\hat{\alpha} = \hat{\kappa}_{2,1}/\hat{\kappa}_3$ (3rd order) & $O(n_k p^2)$ & $O(p^2)$ & $O(p^3)$ & 1  \\
        & \fedrcdni (ours)& None  & $O(n_k p^2)$ & $O(p^2)$ & $O(p^2)$ & 1  \\
        & \fedrcddef (ours)  & OLS on agg. stats (server) & $O(n_k p^2)$ & $O(p^2)$ & $O(p^3)$ & 1  \\
        & \fedrcdx (ours) & OLS on local data (client) & $O(n_k p^3)$ & $O(p^2)$ & $O(p^2)$ & $p-1$   \\
        \bottomrule
    \end{tabular*}}
\end{table}

%
%
\section{Experimental Evaluation} \label{sec:experiments}

\subsection{Setup} \label{sec:setup}

\paragraph{Data generation.}
We sample \er ER-2 DAGs (expected density $0.2$) with edge weights from $U([-1.5,-0.5] \cup [0.5,1.5])$. Exogenous noise is drawn from eleven continuous families (Gaussian Cubed, Pareto, Logistic, Uniform, Laplace, Poisson, Exponential, Student-$t$, Gamma, Exponential Cubed, plus a Mixed setting where each variable draws from a different family). Defaults are $p = 20$, $n = 2000$ observations per client, and $K = 10$ clients. The main grid sweeps $p \in \{10, 20, 30, 50\}$ and $K \in \{1, 5, 10, 20, 50\}$ under horizontal partitioning. Results average 10 random seeds with standard error of the mean.

\paragraph{Baselines.}
Centralised: \directlingam~\citep{shimizu2011}, \hc and \hclingam~\citep{Chen2025}. Federated: \fedishc and its no-deflation variant \fedhc~\citep{fedishc2026}. \fedhc is the third-order analogue of \fedrcdni: same single-pass sort, same $O(p^2)$ server cost, but identification through $\tau^{(3)}$ rather than $\tau^{(4)}$. For a fair comparison, all methods recover edge weights via adaptive Lasso on the aggregated $\hat{\Sig}$ with $\lambda = 0.01\bar{\sigma}$. \fedrcdni and \fedrcdx produce the same orderings as \hclingam and \hc given the same aggregated statistics; we plot the centralised counterparts only when their traces add information.

\paragraph{Metrics.}
\lingam methods primarily recover a causal ordering $\sigma$. A pairwise F1 score between orderings penalises valid orderings, since it compares the estimated $\sigma$ against a single arbitrary topological sort extracted from the gold-standard DAG, even though multiple valid sorts typically exist. We instead evaluate against the partial order induced by the true DAG $\G$. The transitive closure $\G_+$ of $\G$ is the set of pairs $(x_i, x_j)$ for which there exists a directed path from $x_i$ to $x_j$ in $\G$, that is, $x_i$ is an ancestor of $x_j$. Every such ancestral relation $x_i \to x_j \in \G_+$ requires $x_i \prec_\sigma x_j$ in the estimated order; pairs with no ancestral relation~are~excluded.

\begin{definition}[Ancestral Accuracy] \label{def:ancacc}
    Let $\sigma$ be a learned causal ordering and $\G$ a ground-truth DAG with transitive closure $\G_+$. The Ancestral Accuracy is
    \begin{equation} \label{eq:ancacc}
        \AncAcc(\sigma, \G) = \frac{\bigl|\{(x_i, x_j) : x_i \to x_j \in \G_+,\; x_i \prec_\sigma x_j\}\bigr|}{|\G_+|}.
    \end{equation}
\end{definition}

\AncAcc rewards any ordering consistent with the ancestral relations of $\G$. We also report Structural Hamming Distance (SHD) when edge-level error is informative.

\subsection{Reproducibility} \label{subsec:reproducibility}
Algorithms are implemented in Python 3.13 using the \texttt{Flower}\footnote{\url{https://pypi.org/project/flwr/}} framework for federated orchestration. \directlingam relies on the \texttt{lingam}\footnote{\url{https://pypi.org/project/lingam/}} library; \hc, \hclingam, \fedishc, and \fedhc were reimplemented from scratch since no public source code was available. Source code and experiment scripts are at \url{https://github.com/ptorrijos99/FedLiNGAM}. Experiments run on an AMD Ryzen AI 9 HX 370 with 32~GB RAM.

\subsection{Robustness Across Noise Distributions} \label{sec:main_results}
Figure~\ref{fig:noise} aggregates \AncAcc and SHD by noise regime: \emph{symmetric} (Logistic, Uniform, Student-$t$, Laplace) where $\kappa_3 \approx 0$, \emph{mixed} (each variable drawn independently from a different family), and \emph{skewed} (the remaining unimodal families). Per-family breakdowns are deferred to Appendix~\ref{app:per_noise}; results on real-world BN topologies appear in Appendix~\ref{app:bnlearn}.

\begin{figure}[htbp]
    \centering
    \subfigure[\footnotesize \AncAcc $\uparrow$ by noise regime]{
        \includegraphics[width=0.47\linewidth]{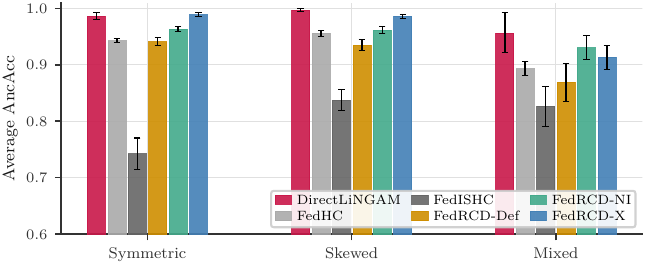}}
    \hfill
    \subfigure[\footnotesize SHD $\downarrow$ by noise regime]{
        \includegraphics[width=0.47\linewidth]{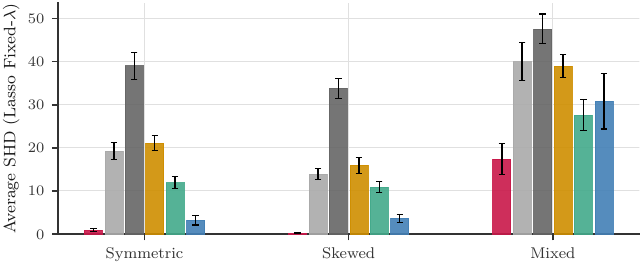}}
    \caption{Performance by noise regime ($p{=}20$, $K{=}10$).}
    \label{fig:noise}
\end{figure}

The hierarchy is consistent across symmetric and skewed regimes, and is the same on synthetic and \texttt{bnlearn} topologies. \fedishc trails substantially, more so on symmetric noise as expected from Section~\ref{sec:pathologies}. \directlingam (centralised) approaches perfect recovery. Among the federated methods, \fedrcdx is very close to \directlingam, with \fedrcdni a few points behind, then \fedhc, then \fedrcddef. The gap between symmetric and skewed regimes is smaller than Proposition~\ref{prop:cdl} alone would predict. The proposition formalises a divergence in the limit $\kappa_3 \to 0$, but finite-sample performance is governed by a different mechanism that we identify~in~Section~\ref{sec:stratification}.
The mixed regime breaks the pattern slightly. \fedishc recovers some accuracy because non-zero average skewness across variables provides partial signal, but it still trails. The federated methods cluster more tightly than under homogeneous noise. \fedrcdni overtakes \fedrcdx by a small margin with \fedhc close behind, while \directlingam does not recover the perfect order in this case. SHD tracks \AncAcc closely throughout.

\subsection{Ablation Study} \label{sec:ablation}

\paragraph{Identification versus deflation.}
To isolate the source of \fedishc's deficit, we implement \fedishcfour, which replaces only the identification criterion ($\tau^{(3)} \to \tau^{(4)}$) while keeping the original third-order deflation~\eqref{eq:ishc_update}. Figure~\ref{fig:scalability:a} reports the all-noise mean. \AncAcc rises from $0.802$ for \fedishc to $0.959$ for \fedishcfour, matching \fedrcdni ($0.960$), which performs no deflation at all. Going from $\tau^{(3)}$ to $\tau^{(4)}$ at fixed no-deflation (\fedhc to \fedrcdni) yields a small but consistent improvement (for $p \leq 20$). At fixed third-order $\hat{\alpha}$-deflation (\fedishc to \fedishcfour) the gain is much larger, because the bottleneck in \fedishc is identification rather than deflation. \fedishcfour matches \fedrcdni because its third-order deflation leaves the fourth-order tensors invariant. \fedrcddef performs genuine fourth-order updates, but compounding finite-sample algebraic noise limits its gains. \fedrcdx remains the best on average, and is the only iterative variant that recomputes cumulants from actual residuals at each round.

\paragraph{Scalability in $p$.}
Figure~\ref{fig:scalability:b} reports mean \AncAcc against $p$ at $K=10$. \directlingam and \fedrcdx are the most stable. \fedishcfour is the best of the rest at $p=10$ but converges with \fedrcdni and \fedhc as $p$ grows, in line with the order-mismatch issue of Section~\ref{sec:pathologies}: the third-order deflation in \fedishcfour leaves $\tau^{(4)}$ unchanged at every iteration, so the algorithm differs from a single-pass ranking only through the active-set restriction in the row sums. At small $p$ this restriction still helps; as $p$ grows, finite-sample noise on the row sums dominates and the two strategies converge.

\paragraph{Federation invariance.}
Figure~\ref{fig:scalability:c} reports \AncAcc against $K$ at $p=20$. All federated methods produce flat curves: cumulant aggregation is lossless given a fixed $\mathcal{D}$, independently of how $\mathcal{D}_k$ are distributed (IID or non-IID; horizontal, vertical, or hybrid). Each entry of the joint cumulant tensor is a function of the bivariate marginal alone, so any pair-covering split yields exactly the same global statistics as the centralised computation, and the curves above transfer literally to the vertical and hybrid regimes.

\begin{figure}[htbp]
    \centering
    \subfigure[\footnotesize \fedishcfour ablation ($p{=}20$, $K{=}10$).]{
        \includegraphics[width=0.25\linewidth]{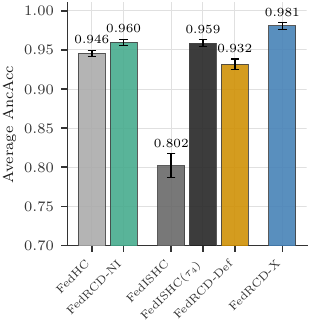}\label{fig:scalability:a}
    }
    \hfill
    \subfigure[\footnotesize Scalability ($K{=}10$).]{
        \includegraphics[width=0.32\linewidth]{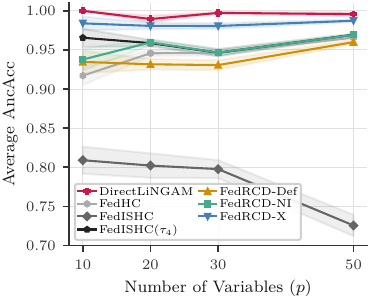}\label{fig:scalability:b}
    }
    \hfill
    \subfigure[\footnotesize Federation invariance ($p{=}20$).]{
        \includegraphics[width=0.32\linewidth]{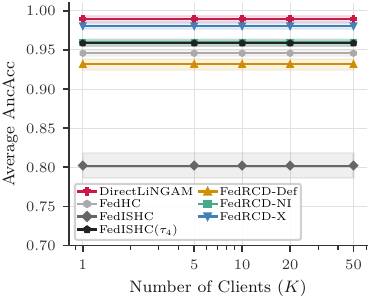}\label{fig:scalability:c}
    }
    \caption{Ablation study on \fedishcfour, scalability, and federation invariance.}
    \label{fig:scalability}
\end{figure}

\subsection{What the cumulant family actually reads} \label{sec:stratification}
The previous sections have argued for \fedrcd on its own merits: it removes a divergence, stabilises deflation, and matches the centralised counterparts of \citet{Chen2025} given the same aggregated statistics. The diagnostic that follows asks a less flattering question. The score $\tau^{(k)}$ is by construction zero in population at true sources, and identification of the order rests on row sums of cumulant tensors. At the sample sizes we actually run, what does that ranking line up with?
\citet{reisach2021nips} define the \emph{varsortability} $v$ of a dataset as the fraction of directed paths along which marginal variance is monotone, and report $v > 0.94$ under generic Additive Noise Model simulations. In that regime, MSE-based continuous methods (NOTEARS, GOLEM-EV, MSE-GDS) attain state-of-the-art recovery on raw data and collapse under marginal standardisation, while \directlingam, PC, and GES are unaffected because they use scale-invariant criteria. The effect is a property of the data scale, not of any specific method.
The same mechanism acts on cumulant magnitudes. Under \lingam, descendants accumulate variance from their ancestors, $\mathrm{Var}(x_i)$ is monotone in DAG depth on average, and $\kappa_m(x_i)$ inherits this scale at every order: empirically $\kappa_m(x_i)$ behaves as $\sigma_i^m$ to leading order. The row sums $\sum_j \tau^{(k)}_{ij}$ that drive identification across the entire cumulant family, federated or centralised, are dominated at finite $N$ by this depth-monotone scale rather than by the population asymmetry, which vanishes for true sources. We refer to the regularity as \emph{variance stratification}; it is the cumulant counterpart of varsortability.
To check whether stratification is what the family reads, we standardise each variable marginally before aggregating cumulants: $\tilde{x}_i = x_i/\sigma_i$. The transform is a diagonal rescaling that preserves \lingam identifiability (the order is unchanged) but kills the variance ladder, taking $v$ from $0.953$ on raw data to exactly $0.5$. Table~\ref{tab:stratification} reports \AncAcc $\pm$ SD in both regimes.

\begin{table}[htbp]
    \caption{Mean \AncAcc on raw and marginally standardised data ($p=20$, $K=10$). Standardisation removes the depth-monotone variance ladder, reducing $v$ to $0.5$. Values below the random baseline of $0.5$ indicate inversion of the true~order.} \label{tab:stratification}
    \centering \footnotesize
    \resizebox{\linewidth}{!}{%
    \begin{tabular*}{1.3\linewidth}{@{\extracolsep{\fill}} l S[table-format=1.3(3), separate-uncertainty] S[table-format=1.3(3), separate-uncertainty] S[table-format=1.3(3), separate-uncertainty] S[table-format=1.3(3), separate-uncertainty] @{}}
        \toprule
        & \multicolumn{2}{c}{\bsc{Symmetric}} & \multicolumn{2}{c}{\bsc{Skewed}} \\
        \cmidrule(r){2-3}\cmidrule(l){4-5}
        \bsc{Method} & {{\bsc{Raw}} ($v\text{=}0.953$)} & {{\bsc{Standardised}} ($v\text{=}0.5$)} & {{\bsc{Raw}} ($v\text{=}0.953$)} & {{\bsc{Standardised}} ($v\text{=}0.5$)} \\
        \midrule
        \directlingam ($K\!=\!1$)      & \bfseries 0.996 \pm 0.002 & \bfseries 0.996 \pm 0.002 & \bfseries 0.998 \pm 0.001 & \bfseries 0.998 \pm 0.001 \\
        \midrule
        \fedhc                         & 0.937 \pm 0.003 & 0.410 \pm 0.014 & 0.933 \pm 0.004 & 0.573 \pm 0.016 \\
        \fedishc                       & 0.769 \pm 0.009 & 0.470 \pm 0.013 & 0.839 \pm 0.008 & 0.436 \pm 0.016 \\
        \fedishcfour                   & 0.940 \pm 0.003 & 0.387 \pm 0.013 & 0.942 \pm 0.004 & 0.521 \pm 0.020 \\
        \fedrcddef                     & 0.924 \pm 0.005 & 0.308 \pm 0.010 & 0.917 \pm 0.007 & 0.358 \pm 0.017 \\
        \hclingam ($K\!=\!1$) / \fedrcdni & 0.940 \pm 0.003 & 0.303 \pm 0.013 & 0.942 \pm 0.004 & 0.441 \pm 0.022 \\
        \hc ($K\!=\!1$) / \fedrcdx       & 0.978 \pm 0.003 & 0.368 \pm 0.014 & 0.976 \pm 0.004 & 0.470 \pm 0.019 \\
        \bottomrule
    \end{tabular*}}
\end{table}

Every cumulant-based method collapses under standardisation, in both regimes, regardless of cumulant order ($\tau^{(3)}$ or $\tau^{(4)}$), of deflation strategy (none, server-side algebraic, or client-side OLS), and of whether the algorithm is centralised or federated. The pattern matches what~\citet{reisach2021nips} reported for MSE-based methods, here extended to a new class of estimators. Standardised \AncAcc falls at or below the conditional-random baseline of $0.5$, with several methods in the $0.30$--$0.40$ range. The strictly below-random performance is a stronger effect than the one observed for NOTEARS, and the mechanism is direct: standardisation rescales $\kappa_m(x_i)$ by $\sigma_i^{-m}$, deeper nodes lose more cumulant magnitude than sources, the depth-monotone scale reverses, and the row-sum ranking confidently reads the inverted ladder. \directlingam is unaffected because its kernel-based independence tests on raw residuals access the population asymmetry directly.

The asymmetry that $\tau^{(3)}$ and $\tau^{(4)}$ encode is present in the data but not accessible to the cumulant family at $N$ in the thousands; aggregation of order-$k$ tensors confines them to the stratification proxy. This is the price of exact federation under cumulant aggregation, and it parallels the price MSE-based methods pay for differentiable acyclicity~\citep{reisach2021nips}. It also explains the raw-data gap between \fedishc and \fedrcdx under symmetric noise. \fedishc's deflation injects the divergent noise of $\hat{\alpha}_{js}$ into~\eqref{eq:ishc_update} and corrupts the depth-monotone scale across the $p-1$ updates. \fedrcdx recomputes cumulants from actual residuals at every round and restores stratification on increasingly clean data. \fedhc and \fedrcdni preserve the original stratification trivially. \fedrcddef sits in between: a stable coefficient, but algebraic updates that still inject finite-sample noise.

%
%
\section{Conclusion} \label{sec:conclusion}
We started from the theoretical failure mode of \fedishc under symmetric noise (the third self-cumulant in the denominator goes to zero) and from the slightly less obvious follow-up: replacing the score is not enough, because the deflation operates on the wrong order. The \fedrcd family fixes both at once by working in the fourth-order domain with a covariance-based OLS coefficient. \fedrcdni and \fedrcdx are exact federated counterparts of \hclingam and \hc, match centralised performance given the same aggregated statistics, apply to horizontal, vertical, and hybrid federation, and support exact unlearning at any granularity in the single-round variants (\fedrcdni and \fedrcddef).

The diagnostic of Section~\ref{sec:stratification} is the part of the paper that surprised us most. The entire cumulant-based federated family, together with its centralised baselines, ranks variables by the variance ladder induced by the DAG rather than by the population asymmetry that the scores encode at zero. Marginal standardisation collapses every cumulant method to near-random ordering, while scale-invariant \directlingam keeps full performance. This places the cumulant family squarely in the scale-dependent regime that~\citet{reisach2021nips} identified for MSE-based methods such as NOTEARS or GOLEM-EV. The family is therefore appropriate when downstream variables accumulate variance from their ancestors, a regime that is structural rather than synthetic, as the Danube river-flow benchmark of~\citet{Chen2025} illustrates. Within it, \fedrcdx reaches the ceiling of exact federation under cumulant aggregation. \fedishc falls below that ceiling under both symmetric and asymmetric noise. Open directions for future work include determining when stratification suffices, finding the sample size required to access population asymmetry, and applying differential privacy to transmitted~tensors.

\section*{Acknowledgements}
This work was supported by PID2022-139293NB-C32 (MICIU/AEI/10.13039/501100011033 and ERDF, EU), FPU21/01074 (MICIU/AEI/10.13039/501100011033 and ESF+), and 2025-GRIN-38476 (Universidad de Castilla-La Mancha and ERDF, A way of making Europe).
Fabio Stella has been supported by the MUR under the grant ``Dipartimenti di Eccellenza 2023-2027'' of the Department of Informatics, Systems and Communication of the University of Milano-Bicocca, Milan, Italy, and by the National Plan for NRRP Complementary Investments (Project n. PNC0000003 - AdvaNced Technologies for Human-centrEd Medicine (ANTHEM)).

\bibliography{references}

%
%
\appendix

%
%
\section{Proofs of Propositions~\ref{prop:cdl} and~\ref{prop:stability}} \label{app:proofs}

\begin{proof}[Proof of Proposition~\ref{prop:cdl}]
    Let $a = \kappa_{2,1}(x_j, x_s)$ and $b = \kappa_3(x_s)$ denote the population values, and let $\hat{a}, \hat{b}$ be their unbiased aggregated estimators with $\mathrm{Var}[\hat{a}] = \sigma_a^2/N$ and $\mathrm{Var}[\hat{b}] = \sigma_b^2/N$. By the delta method~\citep[Ch.~3]{vandervaart1998} for $f(a,b) = a/b$ with $b \neq 0$,
    \begin{equation}
        \mathrm{Var}[f(\hat{a},\hat{b})] \approx \frac{1}{b^2}\,\frac{\sigma_a^2}{N} + \frac{a^2}{b^4}\,\frac{\sigma_b^2}{N} - \frac{2a}{b^3}\,\frac{\sigma_{ab}}{N},
    \end{equation}
    which gives~\eqref{eq:cdl_variance}. Under the \lingam model $a = \beta_{js}^2 b$, so $a^2/b^4 = \beta_{js}^4/b^2$ and all three terms are $O(b^{-2}/N)$. The first term dominates: $\sigma_a^2/(b^2 N)$, where $\sigma_a^2 = N\,\mathrm{Var}[\hat{\kappa}_{2,1}]$ involves joint fourth-order moments of $x_j$ and $x_s$ that remain bounded away from zero even when $\kappa_3(x_s) = 0$, since the joint distribution of $(x_j, x_s)$ retains non-Gaussian fourth-order structure through the mixing. Hence the leading order is $\Theta(\kappa_3(x_s)^{-2}/N)$. The sixth-moment condition guarantees that $\sigma_a^2$, $\sigma_b^2$, and $\sigma_{ab}$ are finite.
\end{proof}

\begin{proof}[Proof of Proposition~\ref{prop:stability}]
    The result is a standard application of the delta method to a well-conditioned ratio; we record it explicitly to enable the contrast with Proposition~\ref{prop:cdl}. Let $\mu = \Sigma_{js}$ and $\nu = \Sigma_{ss} > 0$. Sample covariance estimators satisfy $\mathrm{Var}[\hat{\Sigma}_{js}] = O(N^{-1})$ under bounded fourth moments~\citep[Ch.~3]{vandervaart1998}. The delta method applied to \mbox{$f(\mu,\nu) = \mu/\nu$~gives}
    \begin{equation}
        \mathrm{Var}[\hat{\beta}_{js}] \approx \frac{1}{\nu^2}\,\mathrm{Var}[\hat{\Sigma}_{js}] + \frac{\mu^2}{\nu^4}\,\mathrm{Var}[\hat{\Sigma}_{ss}] - \frac{2\mu}{\nu^3}\,\mathrm{Cov}[\hat{\Sigma}_{js},\hat{\Sigma}_{ss}] = O(N^{-1}),
    \end{equation}
    with constant $C = \sigma_\mu^2/\nu^2 + \mu^2 \sigma_\nu^2/\nu^4 - 2\mu\,\sigma_{\mu\nu}/\nu^3$ bounded for any fixed distribution with $\nu = \mathrm{Var}(x_s) > 0$. No cumulant near zero enters the denominator.
\end{proof}

%
%
\section{Algebraic Cumulant Updates for FedRCD-Def} \label{app:updates}
The server-side variant \fedrcddef applies in-place updates to all aggregated arrays after replacing each variable $x_j$ with the residual $r_j = x_j - \beta_{js}\,x_s$, where $x_s$ is the source identified in the current iteration. Multilinearity of cumulants yields closed-form expressions, extending Lemma~2 of~\citet{fedishc2026} from third to fourth order and from $\hat\alpha$ to the stable~$\hat\beta$:
\begin{align} \label{eq:cov_update}
    \Sigma_{ij}' &= \Sigma_{ij} - \beta_{is}\,\beta_{js}\,\Sigma_{ss}, \\
    \kappa_3(x_j)' &= \kappa_3(x_j) - \beta_{js}^3\,\kappa_3(x_s) - 3\beta_{js}\,\kappa_{2,1}(x_j,x_s) + 3\beta_{js}^2\,\kappa_{1,2}(x_j,x_s), \label{eq:fedrcd_c3} \\
    \kappa_4(x_j)' &= \kappa_4(x_j) - 4\beta_{js}\,\kappa_{3,1}(x_j,x_s) + 6\beta_{js}^2\,\kappa_{2,2}(x_j,x_s) - 4\beta_{js}^3\,\kappa_{1,3}(x_j,x_s) + \beta_{js}^4\,\kappa_4(x_s), \label{eq:fedrcd_c4} \\
    \kappa_{1,2}(x_i,x_j)' &= \kappa_{1,2}(x_i,x_j) - \beta_{is}\,\beta_{js}^2\,\kappa_3(x_s), \label{eq:fedrcd_c12} \\
    \kappa_{2,1}(x_i,x_j)' &= \kappa_{2,1}(x_i,x_j) - \beta_{is}^2\,\beta_{js}\,\kappa_3(x_s). \label{eq:fedrcd_c21}
\end{align}
Equations~\eqref{eq:fedrcd_c3} and~\eqref{eq:fedrcd_c4} use the empirical cross-cumulants directly, without the substitution $\kappa_{2,1}(x_j,x_s) \approx \hat{\alpha}_{js}^2\,\kappa_3(x_s)$ that \fedishc applies in~\eqref{eq:ishc_update}. Equations~\eqref{eq:fedrcd_c12} and~\eqref{eq:fedrcd_c21} invoke the source condition $\kappa_{2,1}(x_s, x_i) \approx \beta_{is}\,\kappa_3(x_s)$ to avoid transmitting trivariate cumulants at $O(p^3)$ cost; the approximation is exact in population when $x_s$ is a true source.

%
%
\section{Extended Experimental Results} \label{app:graphs}

\subsection{Per-Noise-Family Breakdown} \label{app:per_noise}
For completeness, Figure~\ref{fig:noise-appx} reports \AncAcc on each of the eleven individual noise families separately, and Figure~\ref{fig:shd-appx} the corresponding SHD. The patterns of Section~\ref{sec:main_results} hold uniformly: \fedishc is consistently the worst federated method, \fedrcdx the best (except with Mixed noise regime), with \fedrcdni and \fedhc tracking close behind.

\begin{figure}[htbp]
    \centering
    \includegraphics[width=\linewidth]{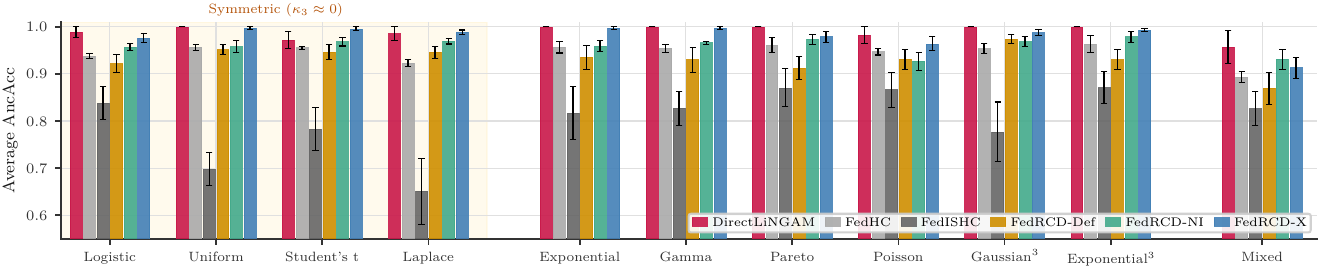}
    \caption{\AncAcc per noise family ($p{=}20$, $K{=}10$).}
    \label{fig:noise-appx}
\end{figure}

\begin{figure}[htbp]
    \centering
    \includegraphics[width=\linewidth]{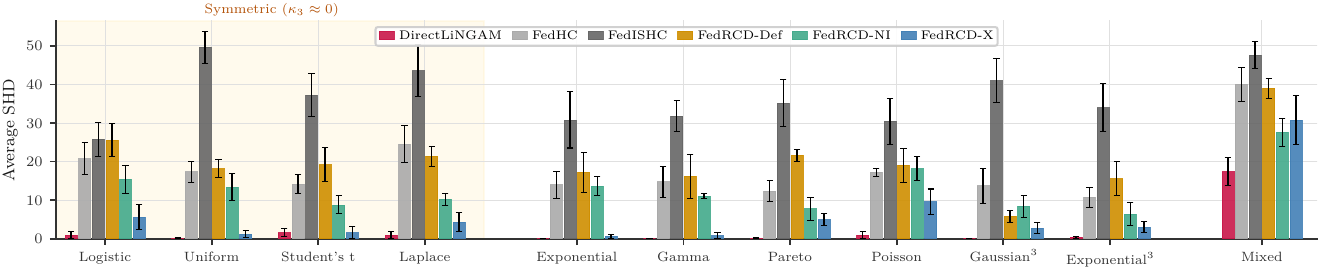}
    \caption{SHD per noise family ($p{=}20$, $K{=}10$).}
    \label{fig:shd-appx}
\end{figure}

\FloatBarrier
\subsection{Real-World Topologies} \label{app:bnlearn}
Figure~\ref{fig:bnlearn} repeats the noise-regime analysis on eight standard networks from the \texttt{bnlearn} BN repository\footnote{\url{https://www.bnlearn.com/bnrepository/}}: \textsc{Asia}, \textsc{Sachs}, \textsc{Child}, \textsc{Insurance}, \textsc{Water}, \textsc{Mildew}, \textsc{Alarm}, and \textsc{Barley}, ranging from $p=8$ to $p=48$. Each topology runs through the eleven noise families with the same edge-weight pipeline as the main experiments. The pattern mirrors Figure~\ref{fig:noise}. \directlingam almost recovers the perfect orders, but its SHD under mixed noise fails slightly. \fedrcdx leads on symmetric and skewed inputs by a clear margin; under mixed noise the federated methods cluster, with \fedhc, \fedrcdni, and \fedrcdx essentially tied.

\begin{figure}[htbp]
    \centering
    \subfigure[\footnotesize \AncAcc $\uparrow$ by noise regime]{
        \includegraphics[width=0.47\linewidth]{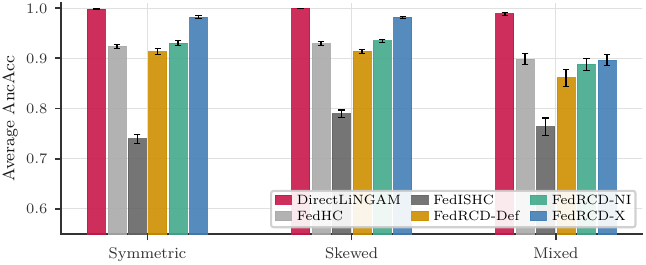}
    }
    \hfill
    \subfigure[\footnotesize SHD $\downarrow$ by noise regime]{
        \includegraphics[width=0.47\linewidth]{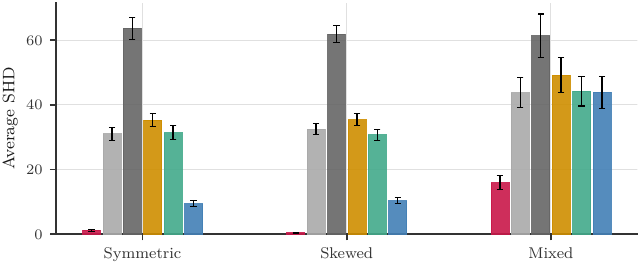}
    }
    \caption{Mean performance on eight \textsc{bnlearn} topologies ($K{=}10$).}
    \label{fig:bnlearn}
\end{figure}

\end{document}